\def\year{2022}\relax
\documentclass[letterpaper]{article} 
\usepackage{aaai22}  
\usepackage{times}  
\usepackage{helvet}  
\usepackage{courier}  
\usepackage[hyphens]{url}  
\usepackage{graphicx} 
\usepackage{natbib}  
\usepackage{caption} 
\DeclareCaptionStyle{ruled}{labelfont=normalfont,labelsep=colon,strut=off} 
\usepackage{algorithm}
\usepackage{algorithmic}
\usepackage{microtype}

\usepackage{newfloat}
\usepackage{listings}
\floatstyle{ruled}
\newfloat{listing}{tb}{lst}{}
\floatname{listing}{Listing}

\usepackage{dsfont}
\usepackage{amsmath,amsthm,amssymb}
\usepackage{todonotes}

\DeclareMathOperator*{\argmin}{arg\,min}

\newcommand{\joint}{P(X,Y)}
\newcommand{\conditional}[1]{P(Y|#1)} 
\newcommand{\constraints}{\mathcal{G}}
\newcommand{\marginal}{P(Y)}

\newcommand{\loss}{\mathcal{L}}
\newcommand{\proxy}{\tilde{\loss}}

\newtheorem{theorem}{Theorem}[section]
\newtheorem{corollary}{Corollary}[theorem]
\newtheorem{lemma}[theorem]{Lemma}
\newtheorem{definition}{Definition}
\newtheorem{proposition}[theorem]{Proposition}

\title{The Perils of Learning Before Optimizing}
\author{
Chris Cameron\textsuperscript{\rm 1},
Jason Hartford\textsuperscript{\rm 2},
Taylor Lundy\textsuperscript{\rm 1},
Kevin Leyton-Brown\textsuperscript{\rm 1}}

\affiliations {
\textsuperscript{\rm 1}Department of Computer Science, University of British Columbia\\
\textsuperscript{\rm 2}Mila, Universit\'e de Montr\'eal\\
\{cchris13, tlundy, kevinlb\}@cs.ubc.ca,
jason.hartford@mila.quebec
}

\begin{document}
\maketitle

\begin{abstract}
Formulating real-world optimization problems often begins with making predictions from historical data (e.g., an optimizer that aims to recommend fast routes relies upon travel-time predictions). Typically, learning the prediction model used to generate the optimization problem and solving that problem are performed in two separate stages. Recent work has showed how such prediction models can be learned end-to-end by differentiating through the optimization task. Such methods often yield empirical improvements, which are typically attributed to end-to-end making better error tradeoffs than the standard loss function used in a two-stage solution. We refine this explanation and more precisely characterize when end-to-end can improve performance. When prediction targets are stochastic, a two-stage solution must make an a priori choice about which statistics of the target distribution to model---we consider expectations over prediction targets---while an end-to-end solution can make this choice adaptively. 
We show that the performance gap between a two-stage and end-to-end approach is closely related to the \emph{price of correlation} concept in stochastic optimization and show the implications of some existing POC results for the predict-then-optimize problem. We then consider a novel and particularly practical setting, where multiple prediction targets are combined to obtain each of the objective function’s coefficients. We give explicit constructions where (1) two-stage performs unboundedly worse than end-to-end; and (2) two-stage is optimal. 
We use simulations to experimentally quantify performance gaps and identify a wide range of real-world applications from the literature whose objective functions rely on multiple prediction targets, suggesting that end-to-end learning could yield significant improvements.
\end{abstract}

\section{Introduction}\label{sec:introduction}

It is increasingly common to face optimization problems whose inputs must be learned from data rather than being given directly. For example, consider a facility location application in which loss depends  on  predictions about both traffic congestion and demand. The most natural way to address such problems is to separate prediction from optimization into two separate stages
\citep{yan2012detecting, centola2018behavior, bahulkar2018community}. In the first stage, a model is trained to optimize some standard loss function (e.g., predict travel times to minimize mean squared error). In the second stage, the model predictions are used to parameterize an optimization problem, which is then solved (e.g., recommend facility locations given demand and travel time predictions).

Recently there has been a lot of interest in the ``end-to-end'' approach of training a predictive model to minimize loss on the downstream optimization task \citep{Donti2017,demirovic2019investigation,Wilder2019,Wilder2020}. 
Advances in deep learning have enabled any set of differentiable functions to be chained together, allowing for the end-to-end training of a differentiable prediction task coupled with a differentiable downstream task using back propagation. \citet{Amos2017} introduced the concept of optimization as a layer, showing how to analytically compute gradients through a QP solver. This sparked a number of follow up papers showing how to build differentiable layers for different classes of optimization problems (e.g., submodular optimization \citep{Djolonga2017}, linear programming \citep{Wilder2019}, general cone programs \citep{Agrawal2019}, and disciplined convex programs (\citeauthor{Agrawal2020}~\citeyear{Agrawal2020})).

Such ``end-to-end'' approaches have been shown to improve performance on a diverse set of applications (e.g., inventory stocking \citep{Amos2017}, bipartite matching \citep{Wilder2019}, and facility location \citep{Wilder2020}). The improvements are often attributed to the end-to-end approach having made better error trade-offs than the two-stage approach. For example, \citet{Donti2017} argue that since all models inevitably make errors, it is important to look at a final task-based objective; \citet{Wilder2019} suggest that end-to-end optimization is likely to be especially beneficial for difficult problems where the best model is imperfect (e.g., when either model capacity or data is limited); and \citet{elmachtoub2020smart} argue for the effectiveness of end-to-end when there is model misspecification and perfect prediction is unattainable. One might therefore assume that end-to-end learning offers no benefit in an ``error-free'' setting where we have access to the Bayes optimal predictor for a given loss function.

This paper shows otherwise: that even when there is no limit on model capacity or training set size, the two-stage approach can fail catastrophically in the common case where the prediction stage models expectations over prediction targets.
In contrast, end-to-end approaches can adaptively learn statistics of the joint distribution that matter most for the downstream task. 
This gives them  a clear advantage when the joint distribution is not well approximated by the product distribution of its marginal probabilities. We formalize this advantage by drawing a connection to the \emph{price of correlation} concept in stochastic optimization \citep{agrawal2012price}. The price of correlation compares the performance of an optimal solution of a stochastic program to the approximate solution that treats every input as independent. We show that in stochastic optimization settings where we can prove that end-to-end is optimal, we can leverage price-of-correlation bounds to obtain worst-case lower bounds on the gap between a two-stage and end-to-end approach. Unfortunately, we also show that end-to-end is not optimal for all stochastic optimization settings.

We then consider a novel setting in which predict-then-optimize is a common paradigm and end-to-end learning is particularly relevant: where multiple prediction targets are combined to obtain each of the objective function's coefficients. We show that this setting can give rise to potentially unbounded performance gaps between the end-to-end and the two-stage approaches.
The key properties we leverage to show a gap are nonlinearity in the objective function's dependence on the prediction targets and correlation structure across targets. All of our results consider a class of problems where end-to-end performs just as well as a stochastic optimization oracle. We make three main contributions: (1)~we establish a base case by showing that when the parameterization is linear, the two-stage approach is optimal; (2)~we construct a distribution where the gap is unbounded for a pairwise product parameterization function that occurs commonly in practice; and (3)~we show how to construct a distribution for any nonlinear parameterization such that there is a gap. 

These results highlight two main features of an optimization problem that together can lead to suboptimal two-stage solutions, and that practitioners should look out for in practice. The first is that multiple targets are learned independently, without accounting for correlations. This is likely to happen when the targets are conceptually different and tend to be learned separately, but may be correlated via latent factors that affect them both. The second is that targets are combined nonlinearly to compute coefficients of the objective function for downstream optimization. This is quite common in practice, notably where the cost of some decision is the product of two inputs (e.g., total cost = cost/unit $\cdot$ \#units). 

Through a running example of the demand-weighted facility location problem, we (1) provide intuition for how such correlations might arise in practice and (2) experimentally explore how correlation impacts the performance gap between a two-stage and end-to-end approach. We design a synthetic benchmark that controls the level of correlation between travel times and demand. We assume that a two-stage approach is able to optimally predict both demand and travel time, but show that it increasingly underperforms relative to the end-to-end approach as the correlation between demand and travel time increases.

Although---particularly given our way of framing the problem---the reader might feel that two-stage learning is obviously problematic in such domains, we illustrate that two-stage learning has been widely used in real-world optimization problems likely to satisfy these conditions. Unfortunately, raw data was unavailable for any of the published real-world application work we reviewed.
We thus support our argument by describing this work where the parameterization involves pairwise products of the targets \citep{mauttone2009route, zhen2016supply, alonso2017predictive,sert2020freight}, and identifying plausible correlation structures between targets that would induce practical gaps between two-stage and end-to-end performance. 

The remainder of this paper is organized as follows. In Section 2, we survey the literature. In Section 3, we give our results characterizing a gap between two-stage and end-to-end in an error-free setting. In Section 4, we perform a simulation study analyzing how correlation impacts the performance gap. In Section 5, we discuss examples from the literature where end-to-end learning could yield improvements. Finally, we summarize our contributions in Section 6.

\section{Related Work}\label{sec:background}
To fully exploit the power of gradient-based learning methods, the end-to-end learning philosophy advocates for jointly learning all the parameters between the raw input and the final outputs of the application. End-to-end models are typically learned with generic architectures that are 
extremely flexible and have been shown to be  effective  for a wide variety of applications.

However, leveraging knowledge about the underlying task can also provide a useful inductive bias so the appropriate structure does not need to be discovered from scratch. In the predict-then-optimize setting, we know that the downstream task involves solving a specific optimization problem. Instead of using a generic method to learn to predict the solution to the optimization problem, the end-to-end learning setup learns the target parameters and relies on existing solvers to find associated solutions. This approach can be optimized end to end with gradient-based methods if we can compute gradients through  the underlying optimization problem. 

There is a significant body of recent work that builds differentiable optimization solvers for various classes of optimization problems. \citet{Amos2017} introduced OptNet, an optimization layer that computes gradients for a quadratic program (QP) solver. \citet{Djolonga2017} and \citet{Tschiatschek2018} showed how to differentiate through submodular optimization. \citet{Barratt2018} and \citet{Agrawal2019} later developed a differentiable layer for solving cone programs and \citet{Agrawal2020} developed a differentiable solver for disciplined convex programs along with a rich implementation leveraging the \texttt{cvxpy} python package. \citet{Wilder2019} showed how to differentiate through linear programs. They added a quadratic regularization term to the objective function to ensure it is everywhere differentiable. \citet{ferber2020mipaal} expanded on this work to build an approximate differentiable layer for solving linear programs with integer constraints (mixed-integer programs). They approximately differentiated through a mixed-integer program by relaxing to a linear program and adding additional constraints via cutting plane methods to push the optimum of the linear program close to the integer-optimal solution. 
Motivated by the fact that running these optimization problems can be computationally expensive, \citet{elmachtoub2020smart} and \citet{Wilder2020} developed less computationally-demanding approaches. \citet{elmachtoub2020smart} developed a cheaper, convex surrogate optimization solver to approximate polyhedral and convex problems with a linear objective. \citet{Wilder2020} found a cheaper proxy optimization problem that is structurally similar to the true optimization problem. They created a differentiable layer for $k$-means clustering as a proxy for solving structurally-related-combinatorial-optimization problems like facility location that are typically solved via linear or integer programs.

These methods developed the machinery for differentiating through optimization routines and demonstrated their benefits experimentally.
Our main contribution is to provide explicit worst-case lower bounds on the performance gap between end-to-end and two-stage solutions. In doing so, we refine the standard explanation for why end-to-end approaches can outperform two-stage methods: most literature talks about end-to-end learning making the right error trade-offs, but here we show that even when one uses the optimal model implied by the first-stage loss, a mismatch between the first-stage loss and the optimization task of interest can lead to unbounded gaps in performance.

With respect to nonlinearity, \citet{elmachtoub2020smart} showed empirical evidence that end-to-end learning performs better relative to a two-stage approach with increasing non-linearity between the features and prediction targets. Our results leverage a different form of nonlinearity where the objective function is a nonlinear function of prediction targets.

Our results build on an insight from the stochastic optimization literature that the gap between the optimal solution to a given stochastic program over correlated variables, and an approximate solution that treats each variable independently, can be bounded in terms of the \emph{price of correlation} \citep{DBLP:journals/siamjo/BertsimasNT04, agrawal2012price}. \citet{agrawal2012price} recognized that in practice, estimating correlations is usually much harder than estimating the mean. They investigated the possible loss incurred by solving a stochastic optimization problem when the correlations are not known. We show for many settings studied in the optimization literature that end-to-end solutions can achieve the stochastic optimum, and as a result, they represent a useful alternative to approximate solutions even in settings with small price of correlation. 

\section{Characterizing Gap Between End-to-End and Two-Stage}\label{sec:theory}

We begin by formally defining the \emph{decision-focused learning} setting \citep[cf.][]{Donti2017}. Let $x\in \mathds{R}^{d}$ be observable inputs (e.g., time of day, weather, etc.) and  $y \in \mathds{R}^{d}$ be outputs drawn from some unknown distribution $\conditional{X}$ (e.g., the distribution of travel times given those observables). 
For any observable input, $x$, we want to optimize decision variables, $z$, to minimize our downstream task loss $\loss(z; x) = \mathbb{E}_{y\sim \conditional{x}}[f(y,z)]$ subject to a set of constraints, $\constraints$: $z^*(x) = \argmin_{z\in \constraints} \loss(z; x)$. In our example, if $y$ represents likely travel times given the current road conditions, $x$, then $z^*(x)$ would be the optimal route recommendation, taking into account uncertainty over $y$. 

Identifying $z^*$ is a stochastic programming problem. Since such problems are intractable in general, it is common practice instead to optimize some proxy deterministic optimization problem $\proxy(z; x) = f(\phi(x),z)$ that replaces the expectation over $P(Y|x)$ with some deterministic prediction $\hat{y} = \phi(x)$.
The hope is that the corresponding optimum $\tilde{z}^*(x) = \argmin_{z\in \constraints} f(\phi(x),z)$ will be similar to the stochastic optimum $z^*(x)$. For example, in the predict-then-optimize setting, one must first precommit to learning some functional of the distribution---usually, $\phi(x) = E[y|x]$---and then proceed with the deterministic optimization on the basis of these predictions. We refer to this as the two-stage approach.
It incurs a loss of, 
\begin{align}
    &\loss(z^*_{\text{two-stage}}; x) = \mathbb{E}_{y\sim \conditional{x}}[f(y,z^*_{\text{two-stage}})] \quad 
\end{align}
on the true stochastic problem, where  $\quad z^*_{\text{two-stage}}(x) =  \argmin_{z\in \constraints} f(\phi(x),z)$.\footnote{Note that there can be multiple optimal solutions $z^{*}_{\text{two-stage}}$ and $z'^{*}_{\text{two-stage}}$
that incur different loss $\loss(z^{*}_{\text{two-stage}}; x) \neq 
\loss(z'^{*}_{\text{two-stage}}; x)$. 
We assume the two-stage approach breaks ties arbitrarily since it has no knowledge of the downstream task. All results are unaffected by this assumption.} Two-stage methods are more computationally efficient because they do not require resolving the optimization problem at each iteration of the learning algorithm, and if they do not result in suboptimal performance in the stochastic optimization, they should thus be preferred.

Of course, one might expect that proxy optima from task-agnostic choices for $\phi(x)$ of this form are unlikely to correspond to the stochastic optimum; indeed, a key contribution  of this paper is formally investigating this intuition. In cases where indeed $\loss(z^{*}, x) < \loss(z^*_{\text{two-stage}}, x)$, a better approach can be selecting $\phi(x)$ with the task in mind. 
The end-to-end approach learns a parameterized model $\phi(x; \theta)$ such that the solutions to the deterministic proxy optimization function $\proxy(z; x)$ minimize the loss of the associated stochastic optimization problem.

Given its parameters $\theta$ and the observed input $x$, the end-to-end approach (1)~outputs some prediction $\phi(x; \theta)$; (2)~given that prediction, determines the optimal decisions $z_{\text{\text{end-to-end}}}(x,\theta)$; and (3) evaluates $z_{\text{\text{end-to-end}}}(x,\theta)$ on the true distribution, updating $\theta$ to minimize the resulting loss (typically by back-propagating through the optimization routine). 
The resulting solution takes the form
\begin{align}
    &\loss(z_{\text{end-to-end}}(x;\theta); x) = \mathbb{E}_{y\sim \conditional{x}}[f(y,z_{\text{end-to-end}})]\label{eqn:1}
\end{align}
where
\begin{align}
    &z_{\text{end-to-end}}(x;\theta) =  \argmin_{z\in \constraints} f(\phi(x; \theta),z) \qquad \\
    &\theta^* = \argmin_\theta  \loss(z^*_{\text{end-to-end}}(x, \theta); x).
\end{align} 
End-to-end approaches represent an attractive middle ground between solving the full stochastic problem and achieving computational tractability via strong independence assumptions.

In what follows, we focus on bounding the performance gap between the two-stage and end-to-end approaches  $\loss(z^{*}_{\text{two-stage}}; x)/ \loss(z^{*}_{\text{end-to-end}}; x)$.
For all of our results, we assume an ``error-free'' setting, where both the end-to-end and two-stage approaches have access to the Bayes optimal prediction for their respective loss functions, such that our results are driven by the choice of target, not by estimation error. The Bayes optimal predictor, $f^* = \argmin_f E[\loss(y, f(x))$] minimizes generalization error for a given loss, $\loss$, with features $x$ and targets $y$. The two-stage approach models the conditional expectation, $f^*(x)= E[Y|x]$, which is Bayes optimal when the loss is mean squared error, $\loss(y,f(x)) = (y-f(x))^2$.

We begin by making an important connection between the performance gap we have just defined and the \emph{price of correlation} introduced by \citet{agrawal2012price}. In words, the price of correlation is the worst-case loss in performance incurred by ignoring correlation in the random variables of a stochastic optimization problem.

\begin{definition}[Price of Correlation]
The price of correlation for a loss function $\loss$ is
$\text{POC} = {\loss(z')}/{ \loss(z^*)},$
where $z^{*}$ is the optimal solution to the stochastic program and  $z'$ is the solution minimizing a proxy stochastic program that makes the simplifying assumption that the random variables $\{y_i:i\in (1, \dots, d)\}$ are mutually independent.
\end{definition}

Making connections to the price of correlation will help us to establish performance gaps between two-stage and end-to-end. We first show that $\loss(z^{*}_{\text{two-stage}}; x)$ is lower bounded by the POC numerator, which implies the following result.

\begin{lemma}\label{lemma:corr_gap}
For any stochastic optimization problem over Boolean random variables $\{y_i:i\in (1, \dots, d)\}$, POC $= \loss(z^{*}_{\text{two-stage}}; x) / \loss(z^{*}; x)$.
\end{lemma}
All proofs appear in the appendix\footnote{For the appendix, please see \url{https://www.cs.ubc.ca/labs/beta/Projects/2Stage-E2E-Gap/}}; roughly, this result follows from the fact that the two-stage approach produces an equivalent proxy stochastic optimization program to the program where all random variables are assumed to be mutually independent.

If we knew that $\loss(z^{*}; x)$ was always equal to $\loss(z^{*}_{\text{end-to-end}}; x)$, we could leverage existing POC results to derive bounds on worst-case performance gaps between two-stage and end-to-end. Unfortunately, even the end-to-end approach can be suboptimal.
\begin{proposition}\label{prop:e2e-not-opt}
There exist stochastic optimization problems for which $\loss(z^{*}; x) < \loss(z^{*}_{\text{end-to-end}}; x)$.
\end{proposition}
Nevertheless, for many stochastic optimization problems with POC results, the end-to-end approach does provably obtain an optimal solution. 

\begin{theorem}\label{thm:e2e-opt}
End-to-end is optimal with respect to the stochastic objective for the problems described in Examples~1,2, and 3 from \citet{agrawal2012price}: (i) Two-stage minimum cost flow, (ii) Two-stage stochastic set cover, (iii) Stochastic optimization minimization with monotone submodular cost function.
\end{theorem}

The key idea behind the proof for each part is proving that for any optimal solution $z^{*}$, end-to-end outputs a deterministic target prediction $y_{z^{*}} \in \mathds{R}^{d}$ such that $\argmin_{z} f(y_{z^{*}},z) = z^{*}$. 

Theorem~\ref{thm:e2e-opt} and Examples~1,2, and 3 from \citet{agrawal2012price} imply the following bounds on worst-case gaps for $\loss(z^{*}_{\text{two-stage}}; x) / \loss(z^{*}_{\text{end-to-end}}; x)$.
\begin{corollary}
In the worst case, $\loss(z^{*}_{\text{two-stage}}; x) / \loss(z^{*}_{\text{end-to-end}}; x)$ is:
\begin{enumerate}
    \item[(i)] $\Omega\left(2^{d}\right)$ for two-stage minimum cost flow;
    \item[(ii)] $\Omega\left(\sqrt{d} \frac{\log (\log (d))}{\log(d)}\right)$ for two-stage stochastic set cover; and
    \item[(iii)] $\Omega\left(\frac{e}{1-e}\right)$ for stochastic optimization minimization with monotone submodular cost function problems.
\end{enumerate}
\end{corollary}

While this corollary shows that we can leverage known results from the POC literature to identify large performance gaps between the end-to-end and two-stage approaches (e.g., setting \emph{(i)}), the POC literature focuses mostly on finding settings such as \emph{(ii)} or \emph{(iii)} where the cost of ignoring correlations grows slowly as a function of the problem size.
POC results also tend to arise in the context of stochastic optimization where learning a prediction model of inputs is not common practice and therefore end-to-end learning is not always applicable \citep{agrawal2012price,lu2015reliable}.

We now turn our attention to a setting not previously studied in the POC literature, in which we will show that correlations can lead to unbounded worst-case performance gaps between the two-stage and end-to-end approaches. This setting is particularly amenable to the ``predict-then-optimize" approach. 
Specifically, we allow for two target vectors $y_1\in \mathds{R}^d$ and $y_2\in \mathds{R}^d$, aiming to capture scenarios where coefficients of the objective function depends on more than one unknown quantity (e.g., demand and travel time) and our objective function takes the form
\begin{equation}\label{eqn:element_wise}
f(y,z)= \sum_{i=1}^{d} \gamma (y_{1_{i}},y_{2_{i}})   f_{i}(z) + C,
\end{equation}
where $f_{i}(z)$ are arbitrary functions of the decision variables $z$ (e.g., monomials of degree one or two for linear (LP) and quadratic programs (QP) respectively) and $\gamma(y_{1_{i}},y_{2_{i}})  : \mathds{R}\times \mathds{R}\rightarrow \mathds{R}$ represent coefficients as a function of prediction targets.

When $\gamma(y_{i,1},y_{i,2}) = y_{i,1} \cdot  y_{i,2}$ (e.g., demand-weighted travel time if $y_1$ and $y_2$ represent demand and travel times, respectively),  it is possible to construct settings where the gap between the end-to-end and two-stage approaches grows without bound in the dimensionality of the problem, $d$. 

\begin{theorem}\label{thm:unbounded}
For $\gamma(a,b) = a\cdot b$, there exists an optimization problem with objective function of the form given by Equation \ref{eqn:element_wise} and distribution $\joint$ such that $\loss(z^{*}_{\text{two-stage}}; x) / \loss(z^{*}_{\text{end-to-end}}; x) \rightarrow \infty$ as $d \rightarrow \infty$.
\end{theorem}

The key idea is to construct a distribution where two-stage wrongly estimates a coefficient to be nonzero when the true expected value is zero. We correlate some pairs $a,b$ where $\gamma$ is applied so that at least one is zero with probability 1, but  the expectation of both $a$ and $b$ is above zero. We set up the optimization problem so that the two-stage approach systematically avoids selecting decision variables with zero cost, causing it to pay an unnecessary cost for every dimension of the problem. Note that Theorem \ref{thm:unbounded} is an existence result and therefore must hold for more general $\gamma$ (e.g., a function of any number of the prediction targets).

Furthermore, the end-to-end solution is optimal with respect to the stochastic objective. 

\begin{lemma}\label{thm:linear-e2e}
$\loss(z^{*}; x) = \loss(z^{*}_{\text{end-to-end}}; x)$ if the objective function is of the form given by Equation \eqref{eqn:element_wise}.
\end{lemma}

We now construct a gap between the two-stage and the end-to-end approaches for any nonlinear $\gamma$. First, we show an obvious implication of the linearity of expectation.

\begin{lemma}\label{lemma:lin}
For all $y,y'\in \mathds{R}$, $\gamma(y,y')$ is linear iff  $\forall \marginal, \mathbb{E}_{\marginal}[\gamma(y,y')] =  \gamma(\mathbb{E}_{\marginal}[y],\mathbb{E}_{\marginal}[y'])$.
\end{lemma}

Intuitively, only when the whole objective function $f$ is linear are we guaranteed that $f$'s expected value is equivalent to $f$ applied to the expected values of each of its inputs. A two-stage approach moves the expectation inside the optimization parameterization since it only models the expected values of the target vector $y$. We show how wrongly estimating the expectation of the parameterization leads to suboptimal decisions.   

\begin{theorem}\label{thm:two_stage_gap}
For any nonlinear $\gamma$, we can construct a distribution $\marginal$, functions $f_{i}(z)$, and constraints $g(z)$ such that $\loss(z^{*}; x) < \loss(z^{*}_{\text{two-stage}}; x)$ for $d\geq 2$.
\end{theorem}

Although we show that the two-stage predict-then-optimize approach can achieve poor performance in general, there do exist special cases in which it is optimal.

\begin{theorem}\label{thm:linear}
$\loss(z^{*}; x) = \loss(z^{*}_{\text{two-stage}}; x)$ if 
the objective function is of the form given by Equation \eqref{eqn:element_wise} and 
either (i) $\gamma$ is a linear function of its inputs; or (ii) $(Y_1 \perp Y_2) | X$ and $\gamma(y_{1_i},y_{2,i}) = y_{1_i} \cdot y_{2_i}$ for all $i$.
\end{theorem}

This result follows directly from linearity of expectation and gives some explanation for the wide use of the two-stage approach in practice (In the appendix, we show more generally that two-stage is optimal for any optimization problem $f(y,z)$ that is linear in $y$.). The first condition of Theorem \ref{thm:linear} is implicitly known in the literature but not stated in this form with multiple prediction targets \cite{elmachtoub2020smart}. 
In Section \ref{sec:applications}, we argue that these conditions for two-stage optimality are unlikely to hold in various cases where two-stage approaches are nevertheless used in practice.

\section{Simulation Study}\label{sec:experiments}

We have seen that even in what one would imagine would be benign  settings---such as a linear program where the coefficients of the decision variables are the product of two predictions, $y_1   y_2$---correlations between target variables can lead to potentially unbounded gaps between the performance of two-stage and end-to-end approaches.
In this section, we contextualize these results by (1) providing intuition for how such correlations might arise in practice and (2) experimentally exploring how correlation impacts the performance gap through a simulation. 

We will consider the example of a delivery company that faces a choice of where to locate its facilities, with the aim of minimizing  average demand-weighted travel time between these facilities and its customers. This can be formulated as an integer program (and potentially relaxed to an LP) of the form $\argmin_{z} \sum_{i,j}\mathbb{E}_{\mathcal{D}}[T_{ij} d_{j}]z_{ij}$ with constraints that ensure (1) at most $k$ locations are chosen for placing facilities and (2) a customer $i$ can only connect to a location $j$ if $j$ is chosen to be a facility. 
Every assignment $z_{i,j}$ between customer $i$ and facility $j$ is weighted by the expected demand-weighted travel time between $i$ and $j$.
This parameterization is broken down into the element-wise product of two conceptually different learning targets: (1)~$T_{i,j}$, the travel times between every pair of endpoints (i.e., time cost between customer and facility); and (2)~$d_{j}$, the demand of every customer $j$.

The magnitude of the performance gap between end-to-end optimization and two-stage optimization will depend on the degree of correlation between these targets $T$ and $d$. If they were independent, then Theorem \ref{thm:linear} would apply and it would be safe to use a two-stage approach, learning separate models to predict $T$ and $d$. However, there are many reasons why these targets might be correlated. For example, the weather may have some effect on demand (e.g., when it is raining, customers' demand for deliveries increases) as well as travel times (e.g., in the rain, congestion tends to increase, both because people drive more cautiously and because accidents nevertheless become more common). If the models for demand and travel times do not condition on the weather, then a two-stage approach will wrongly estimate expected-demand-weighted travel times, potentially leading to suboptimal decisions. One solution is for practitioners to find the right conditioning set to ensure the targets are conditionally independent. However, sources of potential correlation are both abundant and challenging to discover.\footnote{There is a close connection to finding the appropriate adjustment set in the causality literature. If one can specify the causal graph that generates the data, it is possible to find minimal adjustment sets \citep[see][for details]{pearl_2009}.} 

We wanted to experimentally investigate the dependence between correlation and the performance gap given plausible correlation structure, and so designed  a synthetic benchmark for this demand-weighted facility location problem. Our input is a bipartite graph between $n$ customers and $m$ facilities and the output is an assignment of customers to facilities such that only $k$ facilities are connected.  Travel times $T_{i,j}$ and demands $d_{j}$ are sampled from some ground truth distribution $\mathcal{D}$, which we parameterize with $\rho$ that controls correlation structure. We want to learn to output $\hat{T_{ij}}$ and $\hat{d_{j}}$ from samples of the true distribution $\mathcal{D}$ in order to choose assignments $\hat{z}_{i,j}$ that minimize $\sum_{i=1}^{n}\sum_{j=1}^{m}\mathbb{E}_{\mathcal{D}}[T_{ij}d_{j}]\hat{z_{ij}}$ subject to:
\begin{align*}\label{eqn:ilp}
        &\textstyle \sum_{i=1}^{n} z_{ij}=1 \text{ for all } j=1,\ldots,m\\
         &\textstyle \sum_{j=1}^{m} z_{ij}\leq Mx_{i} \text{ for all } i=1,\ldots,n,\quad \textstyle \sum^{n}_{i=1} x_{i} \leq k \\
         & z_{ij} \in \{0,1\} \text{ for all } i=1,\ldots,n \text{ and } j=1,\ldots,m\\
         & x_{0,1} \in \{0,1\} \text{ for all } i=1,\ldots,n.
\end{align*}

\subsection{End-to-End Approach}\label{sec:end_to_end}
For the end-to-end demand-weighted facility location solution, we used a one-layer feed-forward neural network which takes as input a matrix of features for every edge $i$,$j$ between customer $i$ and facility $j$ and outputs a demand-weighted travel-time matrix. During training we used a differentiable optimization layer for solving the linear-programming relaxation of the facility location problem based on the work of \citet{Amos2017} and \citet{Wilder2019}. \citeauthor{Amos2017} showed how to find the gradient of solutions for a class of convex programs with respect to the input parameters (i.e., demand and travel times) by differentiating through the Karush-Kuhn-Tucker (KKT) conditions, which are linear equations expressing the gradients of the objective and constraints around the optimum. The optimal solution to a linear program may not be differentiable with respect to its parameters, so \citeauthor{Wilder2019} add a quadratic smoothing term $||z||^{2}_{2}$ weighted by $\zeta$ to create a strongly concave quadratic program. This allows for the optimal solution to move continuously with changes in parameters. The objective function becomes $\sum_{i=1}^{n}\sum_{j=1}^{m}\mathbb{E}_{\mathcal{D}}[T_{ij}d_{j}]z_{ij} + \zeta||z||^{2}_{2}$. 

We implemented this smoothed differentiable version of a linear program with the encoding described earlier using the \texttt{cvxpylayers} package \citep{Agrawal2020}. To find an integer solution at test time, we solved this program optimally with a mixed-integer solver without the smoothing term. 
Since the feed-forward network and the quadratic program are all differentiable with respect to their inputs, we could train the system end-to-end with backpropagation. Following \citet{Wilder2019}, we defined the loss of the network to be the solution quality of the customer-facility assignments output from the linear program $\hat{z}$ given the ground truth demand-weighted distances $c$. By the objective function, the loss is then $c^{T} \hat{z}$.

We used the ADAM optimizer with a learning rate of 0.01 and performed 500 training iterations for each experiment. We set our quadratic penalty term $\zeta$ to be 10. We trained on an 8-core machine with Intel i7 3.60GHz processors and 32 GB of memory and an Nvidia Titian Xp GPU. To get integer solutions at test time, we used the mixed-integer solver GLPK in the \texttt{cvxpy} python package. Training times were very small. Each end-to-end model took a few minutes to train with each MIP requiring less than a second to find the optimal solution.

\subsection{Experimental Setup}

We now construct a simple graph distribution to simulate our working example. The principle behind the distribution is that there are some routes between customers and facilities whose travel times correlate with customer demand (e.g., inclement weather increases both demand and travel times within a city centre) and some that are independent (e.g., weather does not affect travel time on freeway routes). We model some latent variable $l$ (e.g., weather) that correlates demand and travel times (e.g., demand and travel times will both be higher with inclement weather). Every customer has a correlation between demand and travel time controlled by some parameter $\rho$ for half of the facilities $F_{1}$ and zero correlation for the other half of facilities $F_{2}$. The parameter $\rho$ controls the correlations that the latent factor $l$ induces between $T_{i,j}$ and $D_i$ for the facilities in $F_1$. When $\rho$ is 0, there is no correlation between the two; with $\rho=1$, they are perfectly correlated; and when $\rho=-1$, they are perfectly anticorrelated. Our distribution $\mathcal{D}$ is as follows,
\begin{align*}
\mathcal{D} = \begin{cases}
    &l_{i} \sim \mathcal{N}(0, \rho^2 \sigma_{li}^2)\\
    &T_{i,j} \sim  \mathcal{N}(\mu_{i,j} + |l_i|, (1-\rho^{2}) \sigma_{t}^2) \quad\forall j \in F_{1}\\
    &T_{i,j} \sim \mathcal{N}(\mu_{i,j},  \sigma_{t}^2) \quad\forall j \in F_{2}\\
    &D_{i} \sim \mathcal{N}(l_{i} + \mu_{i}, \sigma_{d}).
    \end{cases}
\end{align*}
In our experiments, we set $m,n=20,\sigma_{li}=20.25$, $\sigma_{d}=\sigma_{t}=3$, $\forall j \in F_{1}, \mu_{i,j}=5.5$,   $\forall j \in F_{2},\mu_{i,j}=6, K=1$. We evaluated three approaches; our end-to-end approach described above, a two-stage approach, and an OPT approach. The two-stage approach is Bayes optimal for mean-squared error. It selects assignments based on sample means of training samples for every customer demand and customer-facility travel time. Knowing the expected value for every (demand, travel time) product is sufficient for optimal performance, therefore we defined OPT as selecting assignments based on sample means from the training set of every (demand, travel time) product. We found an integer-optimal solution for every approach at test time.
We ran an experiment for 11 values of $\rho$ between -1 and 1 at equal intervals, which we evaluated on 10 random seeds. For each evaluation, we generated 1000 samples from our distribution and left out 200 as test data. Please see \url{https://www.cs.ubc.ca/labs/beta/Projects/2Stage-E2E-Gap/} for our code and data generation process.

\subsection{Results}

\begin{figure}[t]
    \centering
    \includegraphics[trim={0cm 0.1cm 0 0},clip,width=\columnwidth]{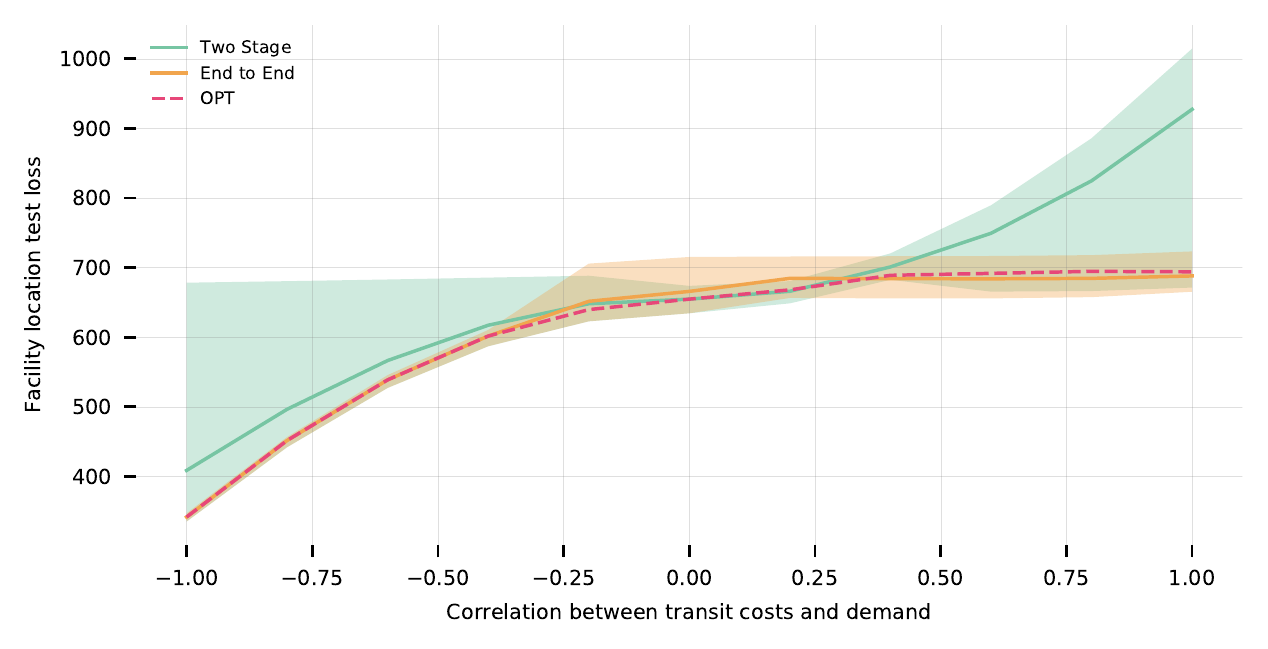}
    \caption{Facility location test performance vs correlation structure with samples from $\mathcal{D}$. End-to-end refers to the method described in Section \ref{sec:end_to_end}; two-stage approach learns sample means for every demand and travel time; and OPT learns sample means for every (demand, travel time) product. The shaded areas represent 90th percentiles.}
    \label{fig:performance}
\end{figure}

Our results are shown in Figure \ref{fig:performance}. The end-to-end approach performed equally as good as OPT across all values of $\rho$. The performance gap grew substantially with increasing positive correlation towards $\rho=1$. With positive correlation, the two-stage approach substantially underestimated the demand-weighted travel times to $F_{1}$ facilities and tended to incorrectly choose to place the facility to be in $F_{1}$. As we expect by Theorem \ref{thm:linear}, there was no gap when $\rho=0$ as demands and travel times are independent of one another. With increasing negative correlation towards $\rho=-1$, we did not see much change in the gap. Negative correlation caused the two-stage approach to substantially overestimate the demand-weighted travel times but the two-stage approach tended to pick the correct facility. This illustrates an important point for understanding when gaps occur in practice. Incorrect estimations are not enough for bad decisions. The estimation must be wrong in a direction that causes a change in decision.
The large confidence intervals of the two-stage approach towards the extremes at $\rho=-1$ and $\rho=1$ suggest that two-stage is also more sensitive to noise in the distribution. An end-to-end approach will recognize that the value of $F_{1}$ facilities relative to $F_{2}$ facilities will grow as $\rho \rightarrow -1$, and as a result, it is more robust to estimation error. With infinite training samples, the two-stage approach will make the correct decision for $\rho \leq 0$, but as $\rho \rightarrow -1$, the margin of error stays constant while wrong decisions become more consequential.

\section{Two-Stage Optimization with Potentially Correlated Targets in Real Applications}\label{sec:applications}

While we have seen that large performance gaps can arise in a synthetic setting, we would of course like to understand whether a similar phenomenon arises in real applications. Unfortunately, we are not practitioners ourselves, and we were not able to identify any publicly available dataset giving the raw inputs to a two-stage optimization approach. We observed that the relevant literature tends to concentrate on the optimization part of the two-stage optimization problem rather than publishing raw data underlying the prediction phase. 

In the absence of such a dataset, we sought other evidence that two-stage optimization was being deployed in practical settings where our theoretical results suggest that end-to-end optimization could be expected to achieve better performance. We identified a wide range of  applications from the literature in which a two-stage approach was used and where there exists potential for correlation between multiplied targets. The examples we found came from a range of different optimization problems, but all involved products of the demand and cost predictions on different travel routes (for buses, taxis, supply chains, and product delivery), and all failed to account for latent factors that might induce correlation in these targets, such as weather or seasonality. Of course, without access to the original data, we cannot be certain that accounting for these latent effects would have resulted in significant changes to solution quality, but they serve as useful examples to show how such correlations can potentially occur. Furthermore, we were struck by how many such examples we were able to discover. Overall, we hope that our work will be taken as a call to arms for the community to explore end-to-end learning on these applications, and that it will  encourage practitioners to publish raw data underlying both prediction and optimization phases.

We now survey the specific work we identified.
Line-planning problems involve taking the products of travel time and demands for bus routes \citep{mauttone2009route}. Demand and travel time can be correlated through precipitation as increasing precipitation has been found to decrease bus ridership  \citep{zhou2017impacts} and increase travel time \citep{tsapakis2013impact}.
In ride-sharing programs, the weights of the matching problem that assigns drivers to customers involves taking the product of the probabilities that a future request is real and the times to serve that request \citep{alonso2017predictive}. Travel times are similarly affected by weather and taxi demand tends to decrease with rainfall during non-rush hours but increase with rainfall during rush hours \citep{chen2017impact}. 
We could not find any papers on either ride sharing or line planning whose authors explicitly conditioned on weather to render these targets conditionally independent, so our analysis indicates that end-to-end approaches likely offer scope for improved solutions.

In supply chain planning, companies optimize routes across supplier stages to customers. For example, \citet{zhen2016supply} investigate supply chain planning for automobile manufacturers. Route costs are proportional to the product of both customer demands and transportation costs between stages, which can be correlated through seasonality. \citet{irum2020sales} show that demand for automobiles peaks in the spring and fall and \citet{van2020disruption} show it is more difficult to procure freight during peak season and transportation delays (and higher costs) are more common.

Finally, demand-weighted facility location problems place facilities to respond to customers and minimize transportation costs; as such, the cost of each route involves the product of a prediction for customer demands and the transportation costs from facilities to customers \citep{sert2020freight}. Transportation delays are higher during peak season \citep{van2020disruption} and most product demands will have seasonal effects.
\citeauthor{sert2020freight} use anonymized data so it is difficult to speculate on specific sources of correlation, but they do not explicitly discuss controlling for these effects.

\section{Conclusions}\label{sec:conclusions}

This paper presented lower bounds on the worst-case difference between two-stage and end-to-end approaches. First, we derive bounds by drawing a connection to \emph{price-of-correlation} results in stochastic optimization. Then, we prove results for a practical setting that apply to benign sources of non-linearity: taking the product of two predictions in an otherwise linear optimization problem is sufficient to construct gaps that grow with the dimensionality of the problem. We identify a number of applications with coefficients of this form where end-to-end solutions may offer large improvements in practice.

\newpage
\section*{Acknowledgements}

This work was funded by an NSERC Discovery Grant, a DND/NSERC Discovery Grant Supplement, a CIFAR Canada AI Research Chair (Alberta Machine Intelligence Institute), a Compute Canada RAC Allocation, a GPU grant from Nvidia, awards from Facebook Research and Amazon Research, and DARPA award FA8750-19-2-0222, CFDA \#12.910 (Air Force Research Laboratory).

\newpage
\bibliography{references}
\newpage
\appendix
\onecolumn
\section{Appendix} 
\begin{proof}[Proof of Lemma \ref{lemma:corr_gap}]
We need to show for any stochastic optimization problem, POC $\leq \loss(z^{*}_{\text{two-stage}}; x) / \loss(z^{*}; x)$.

By the definition of the price of correlation, POC $=\loss(z^{'}; x)/ \loss(z^{*}; x)$, where $z^{'}$ is the optimal solution to a proxy stochastic program over Boolean variables that makes the assumption that all random variables are mutually independent. This is equivalent to reducing all Boolean variables to their marginals.

$z^{*}_{\text{two-stage}}$ is the optimal solution to a proxy stochastic program where all random variables are reduced to their marginal expectation $E[y|x]$, which is equivalent for mutually independent Boolean variables $y$. 
Therefore  $z^{*}_{\text{two-stage}}$ and $z^{'}$ are
optimal solutions to the same proxy stochastic program. In case there are multiple optimal solutions, we assume $z^{*}_{\text{two-stage}}$ and $z^{'}$ are found using the same tie-breaking scheme and therefore have the same loss, $\loss(z^{'}; x) =\loss(z^{*}_{\text{two-stage}})$. Therefore, 
 POC $= \loss(z^{*}_{\text{two-stage}}; x) / \loss(z^{*}; x)$.
\end{proof}

\begin{proof}[Proof of Proposition \ref{prop:e2e-not-opt}]
First, we provide some intuition for the proof. We define the end-to-end approach to be restricted to outputting a deterministic prediction vector $y$ for the downstream problem. We construct an example where this $y$ vector represents probabilities of two Boolean events. For end-to-end to be optimal, it must output marginal probabilities $y$ that elicit the optimal solution on the downstream task. The proof shows that there is no such setting of $y$ to elicit the optimal solution. We construct an example where to be optimal, end-to-end would need to output marginal probabilities for events 1 and 2 such that only two scenarios have positive probability (1) when both events happen and (2) when neither event happens. The only way to assign positive probability on both of these scenarios is to also assign positive probability on the two scenarios where only one event happens. This causes the end-to-end approach to find the wrong solution.

Formally, we need to show there exist stochastic optimization problems such that $\loss(z^{*}; x) < \loss(z^{*}_{\text{end-to-end}}; x)$. We construct a setting where $\loss(z^{*}; x)$ is a $C$ factor better for any constant $C>2$.

Let $b_{1}$,$b_{2}$ be Boolean events of a set $S$ and $E_{S\in D}[f(z,S)]$ be a stochastic optimization problem. The two-stage approach returns the optimal solution given the marginal probabilities $p_{1}$,$p_{2}$ of events $b_{1},b_{2}$ occurring. The end-to-end approach sets $p_{1},p_{2}$ arbitrarily to produce a corresponding $z$ with best downstream loss:

\begin{equation*}
    \loss(z^{*}_{\text{end-to-end}}; x) = \min_{p_{1},p_{2}} E_{S \in D}[f(\argmin_{z}E_{S' \in p_{1},p_{2}}[f(z,S'] ,S)]
\end{equation*}

We define the following distribution over $b_{1},b_{2}$:

\begin{align*}
  D &= \begin{cases}P(b_1=T, b_2=T) = 0.5\\
     P(b_1=F, b_2=F) = 0.5\\
     otherwise~0.
    \end{cases}
\end{align*}

Let $f(z,S)$ have the following cost matrix:

\begin{center}
\begin{tabular}{l|cccc}
& $S_{1}$ & $S_{2}$ & $S_{3}$ & $S_{4}$\\
& $b_{1}=T,b_{2}=T$ & $b_{1}=F,b_{2}=F$ & $b_{1}=T,b_{2}=F$ & $b_{1}=F,b_{2}=T$\\
\hline
$z^{*}$& $1$ & $1$ & $\infty$ & $\infty$\\
$z_{1}$& $\epsilon$ & $C$ & $\epsilon$ & $\epsilon$\\
$z_{2}$& $C$ & $\epsilon$ & $\epsilon$ & $\epsilon$
\end{tabular}
\end{center}

$z^{*}$ is optimal for $C>2$. $z_{1}$ and $z_{2}$ both cost an $O(C)$ factor greater than $z^{*}$. We prove that end-to-end must select either suboptimal solutions $z_{1}$ or $z_{2}$ regardless of how $p_{1}$ and $p_{2}$ are set. There are two cases:
\begin{enumerate}
    \item If end-to-end sets $p_1$ and $p_2$ such that all probability is on $S_{1}$ or $S_{2}$, $z^{*}$ will not be chosen since $z_{1}$ is best for $S_{1}$ and $z_{2}$ is best for $S_{2}$.
    \item If end-to-end sets $p_1$ and $p_2$ such that there is non-zero probability on both $S_{1}$ and $S_{2}$, there must also be non-zero probability on $S_{3}$ and $S_{4}$ because $p_1$ and $p_2$ are marginal probabilities. Since $z^{*}$ has infinite cost for $S_{3}$ and $S_{4}$, $z_{1}$ and $z_{2}$ will always look better to end-to-end than $z^{*}$.
\end{enumerate}
\end{proof}

\begin{proof}[Proof of Theorem \ref{thm:e2e-opt}: Optimality of end-to-end for (i) two-stage minimum cost flow]

This setting is from Example 1 of \citet{agrawal2012price}. There are $n$ Boolean events with marginal probabilities $p_{1}, ..., p_{n}$. Each event corresponds to the existence a sink node with unit demand. Each sink node connects to a single hub node. The optimization problem is to buy capacity for the hub node such that all demand is satisfied. Before seeing the realization of sink nodes, in the first stage you can buy capacity $z$ at price $c^{I}(z)$. Once you have seen the realization of sink nodes, you then must buy enough extra capacity $x$ at higher cost $c^{II}(x)$ to satisfy the demand requests. The goal is to optimize $z$ to minimize expected costs. The problem can be written formally in a single stage as follows:
    
\begin{equation*}
    \argmin_{z} c^{I}(z) + \sum_{i=z}^{n} P\left((\sum_j y_{j}) = i\right)c^{II}(n-i)
\end{equation*}
s.t., $y_{1},...,y_{n}$ are the realization of Boolean events, and $c^{II}(k) > c^{I}(k), \forall k$\\
A two-stage approach wrongly estimates $\sum_{i=z}^{n} P((\sum_j b_{j}) = i)$ by multiplying marginals together to compute probabilities of sets of events occurring. 
We now prove $\loss(z^{*}_{\text{end-to-end}};x) = \loss(z^{*};x)$ for this setting:\\

The end-to-end approach can set $p_{1}...p_{n}$ arbitrarily. For any given $z^{*}$, it is enough to construct $p_{1}...p_{n}$ such that $z^{*}_{\text{end-to-end}} = z^{*}$. Idea is to set $p_{1}...p_{n}$ to make the two-stage costs sufficiently high for $z>z^{*}$ and low for $z<z^{*}$ so $z^{*}$ is optimal.
    
We can put all probability on a single set $S$ such that $p_{1},...,p_{z^{*}} = 1$ and $p_{z^{*}+1},...,p_{n}=0$. That way, we pay zero second-stage cost for increasing $z$ above $z^{*}$ since we've satisfied all demand, but add first-stage costs. If we decrease $z$ by $k$, we must pay $c^{II}(k)$ second-stage cost but only save $c^{I}(k)$. Since $c^{II}(k) > c^{I}(k)$, decreasing $z$ also increases costs. Since any change in $z$ increases costs,  $z^{*}_{\text{end-to-end}} = \argmin E_f(z,S) = z^{*}$.

\end{proof}
    
\begin{proof}[Proof of Theorem \ref{thm:e2e-opt}: Optimality of end-to-end for (i) two-stage stochastic set cover]

This setting is from Example 2 of \citet{agrawal2012price}. There are $n$ Boolean events with marginal probabilities $p_{1}, ..., p_{n}$. Each event corresponds to an item in a set $V$. There are $k$ disjoint subsets whose union is $V$. The optimization problem is to buy subsets to cover realized items. Before seeing the realization of events, in the first stage you can buy any subset $z_{1},...,z_{k}$ at price $c^{I}$ each. Once you have seen the realization, you then must pay a higher cost $c^{II}$ to buy the uncovered subset with the maximum number of realized items. The goal is to optimize $z$ to minimize expected costs. The problem can be written formally in a single stage as follows:

\begin{equation*}
    \argmin_{z} \sum_{i} c^{I}z_{i} + c^{II}\max_{i=1..k} \Bigg\{ \sum_{S \in S_{i}}P(S)|S| \cdot (1-z_{i}) \Bigg\}
\end{equation*}
s.t. $c^{I} < c^{II}$\\
$S$ is subset of items in disjoint subset $S_{i}$. $i$ refers to $i$th of $k$ disjoint subsets, whose union is $S$.
We now prove $\loss(z^{*}_{\text{end-to-end}};x) = \loss(z^{*};x)$ for this setting:\\

For any given $z^{*}$, we need to construct $p_{1}...p_{n}$ such that $z^{*}_{\text{end-to-end}}= z^{*}$. For a given $z^{*}$, we can put all probability on a single set $S$ by assigning probability 1 to all events covered by $z^{*}$ and probability 0 to all events uncovered. $z^{*}$ can never be improved by adding sets, since all events are already covered and there is only one set that can cover each event (i.e., disjoint subsets). Therefore, $z^{*}$ will always be optimal as long as $c^{I}$ (the cost saved by uncovering a set in first stage) is less than $c^{II}|S|$ (the cost of covering in second stage) for smallest $S$ that is covered. Every covered set has a least one event by our construction, therefore uncovering a set would incur a cost of at least $c^{II}$ but only save $c^{I}$. Since $c^{I} < c^{II}$, we can never reduce costs by reducing the covered sets in $z$. Therefore, $z^{*}_{\text{end-to-end}} = \argmin f(z,S) = z^{*}$.
\end{proof}

\begin{proof}[Proof of Theorem \ref{thm:e2e-opt}: Optimality of end-to-end for (i) stochastic optimization with monotone submodular cost function]

This setting is from Example 3 of \citet{agrawal2012price}. There are $n$ Boolean events with marginal probabilities $p_{1},...,p_{n}$. Let the realized set of Boolean events be $S$.
Suppose you have a monotone submodular cost function $c(S)$. You want to decide whether to (1) pay the expected cost of the subset $E_{S}[c(S)]$ or (2) pay a constant cost $C$.

\begin{align*}
    \argmin_{z} & ~c(S)z_{1} + Cz_{2}\\
    &z_{1} + z_{2} = 1, z_{1},z_{2} \in \{0,1\}
\end{align*}

We now prove $\loss(z^{*}_{\text{end-to-end}};x) = \loss(z^{*};x)$ for this setting:\\

Given marginal probabilities $p_{1},...,p_{n}$, the expected subset cost is $\sum_{S} P(S)c(S)$ where $P(S)$ is the product of probabilities for all events in that set $S$. For optimality, end-to-end needs to find $p_{1},...,p_{n}$ such that $\sum_{S} P(S)c(S) = E_{S}[c(S)]$. Let $C^{*}=E_{S}[c(S)]$. First, there must be some set of events $S$ that costs $\geq C^{*}$, since $C^{*} \leq max_{S} c(S)$. Second, among the sets that satisfy this property, by monotonicity there must be some set $S'$ and event $e$ such that $c(S'-\{e\}) \leq C^{*}$ since $C^{*} \geq min_{S} c(S)$. Setting aside event $e$, we set probabilities to 1 for events in $S'$ and 0 otherwise. For event $e$, set $p_{e} \in [0,1]$ such that $p_{e}c(S') + (1-p_{e}c(S'-\{e\}) = C^{*}$. (i.e., linear interpolation between $S'$ and  $S'-\{e\}$). There must be a solution for $p_{e}$ since $c(S'-\{e\})\leq C^{*}\leq c(S')$

\end{proof}

\begin{proof}[Proof of Theorem \ref{thm:unbounded}]

We first define the optimization problem and distribution that we will use to construct our gap. Our construction only depends on the distribution over $Y$, so without loss of generality, we let $\conditional{X} = \marginal$; an analogous construction could be made for every $x$.
We set the optimization problem to be,
\begin{align*}
    &\min  \mathbb{E}_{\marginal} [C + \sum_{i=1}^{d} (y_{i,1}  y_{i,2}   z_{i,1}) + (y_{i,3}  y_{i,4}   z_{i,2})]\\
    & \text{subject to } \sum_{i=1}^{d} z_{i,1} + z_{i,2} \geq d,  \quad 
    \forall i \quad 1\geq z_{i,1}, z_{i,2}\geq 0,
\end{align*}

where $y\in \mathds{R}^{d\times 4}$ and $z\in \mathds{R}^{d\times 2}$. Note that any basic solution to this optimization problem must set at least half of the $z$ variables to be equal to $1$.

Now define some large positive $N \in \mathds{R}$ and some small $\epsilon>0$. Let $P(Y)=P(Y_1,Y_2)P(Y_3,Y_4)$ be distributed as follows. 
\begin{align*}
  P(Y_1,Y_2) &= \begin{cases}0.5\quad y_{i, 1} = 0,  y_{i, 2} = N\\
     0.5\quad y_{i, 1} = N,  y_{i, 2} = 0
    \end{cases}\\
   P(Y_3,Y_4) &= 
    \begin{cases}1. \quad y_{i, 3} = y_{i, 4} = \frac{N}{2}-\epsilon
    \end{cases}
\end{align*}

The intuition behind this distribution is that  $\forall i$, $y_{i,1}  y_{i,2} = 0$ and $y_{i,3}  y_{i,4} > 0$ despite the fact that $y_{i,1}$ and $y_{i,2}$ have higher expected values than $y_{i,3}$ and $y_{i,4}$.

Notice that this means half of our $z$ variables have coefficients of $0$ and we can satisfy our constraints without raising the objective. Therefore, the optimal end-to-end solution is to set $z$ to be $0$ anytime the product of targets is nonzero in expectation; i.e.,
$z^{*}_{\text{end-to-end}}$ sets $ z_{i,1}=1, z_{i,2}=0$, for all $ i$.

The loss obtained by choosing $z^{*}_{\text{end-to-end}}$ is,
\begin{align*}
\loss(z^{*}_{\text{end-to-end}};x) &= \mathbb{E}_{\marginal} [C + \sum_{i=1}^{d} (y_{i,1}  y_{i,2}   1.0) +(y_{i,3}  y_{i,4}   0)]=C.
\end{align*}
The last equality follows from the fact that $\forall i, y_{i,1}  y_{i,2} = 0$.

However, as mentioned above $y_{i,1}$ and $y_{i,2}$ have higher expected values than $y_{i,3}$ and $y_{i,4}$. Therefore, the optimal two-stage solution is to set $z$ to be $1$ for each pair with the lower individual expected values i.e. $z^{*}_{\text{two-stage}}$ sets $z_{i,1}=0, z_{i,2}=1$ for all $i$.

The loss obtained by choosing $z^{*}_{\text{two-stage}}$,
\begin{align*}
\loss(z^{*}_{\text{two-stage}};x) &= \mathbb{E}_{\marginal} [C + \sum_{i=1}^{d} (y_{i,1}  y_{i,2}   0.0) + (y_{i,3}  y_{i,4}   1.0)]\\
&= C + \sum_{i=1}^{d} (\frac{N}{2}-\epsilon)^2
  1.0 \approx \frac{d  N^{2}}{8}.
\end{align*}

Therefore, the multiplicative gap $\left(\frac{\loss(z^{*}_{\text{two-stage}}; x)}{\loss(z^{*}_{\text{end-to-end}}; x)}\right)$ is $d  N^{2}$ up to constant factors.

\end{proof}

\begin{proof}[Proof of Lemma \ref{thm:linear-e2e}]
We set end-to-end to output $y'|x$ such that $y'_{i,1} = \mathbb{E}_{y\sim \conditional{x}}[\gamma(y_{i,1},y_{i,2})]$ for $i=1,\ldots,d$ and $y'_{i,2}=1$ for $i=1,\ldots,d$.

$\loss(z^{*}_{\text{end-to-end}};x) = \min_{z} \sum_{i=1}^{d}  \gamma(y'_{i,1}, y'_{i,2}) \cdot f_{i}(z)$

We know that 
\begin{align*}
  \loss(z^{*};x) &= \min_{z} \mathbb{E}_{y\sim \conditional{x}}\bigg[ \sum_{i=1}^{d} \gamma (y_{i,1},y_{i,2}) \cdot f_{i}(z)\bigg]\\
    &=\min_{z} \sum_{i=1}^{d}  \mathbb{E}_{y\sim \conditional{x}}[ \gamma (y_{i,1},y_{i,2}) \cdot f_{i}(z)]\\
    &=\min_{z} \sum_{i=1}^{d}  f_{i}(z) \cdot \mathbb{E}_{y\sim \conditional{x}}[ \gamma (y_{i,1},y_{i,2})]\\
\end{align*}

By the definition of $\gamma$
\begin{align*}
    \loss(z^{*};x) &= \min \sum_{i=1}^{d/2}  f_{i}(z) \cdot \gamma (\mathbb{E}_{y\sim \conditional{x}}[y_{i,1}, y_{i,2}], 1)\\
    &= \min \sum_{i=1}^{d/2}  f_{i}(z) \cdot \gamma(y'_{i,1}, y'_{i,2})\\
    &= \loss(z^{*}_{\text{end-to-end}};x)
\end{align*}

\end{proof}

\begin{proof}[Proof of Lemma \ref{lemma:lin}]

A function $\gamma(y,y')$ is linear if and only if $\forall \conditional{x}, \mathbb{E}_{\mathcal{D}}[\gamma(y,y')] =  \gamma(\mathbb{E}_{\mathcal{D}}[y],\mathbb{E}_{\mathcal{D}}[y'])$. 

The reverse direction follows by linearity of expectation so we must show that if a function $\gamma(y)$ is nonlinear then $$\exists \mathcal{D}, \mathbb{E}_{\mathcal{D}}[\gamma(y)] \neq  \gamma(\mathbb{E}_{\mathcal{D}}[y]).$$ A function $\gamma$ is linear if and only if $\gamma(\alpha y_1 + (1-\alpha) y_2) = \alpha \gamma(y_1) + (1-\alpha)\gamma(y_2)$ $\forall y_1, y_2, \alpha$. If a function is nonlinear then by definition there must exist two points $y_1,y_2$ between which the function is nonlinear. This means there exists a point $y'= \alpha y_1 + (1-\alpha)y_2$ such that $\gamma(y')\neq \gamma(\alpha y_1 + (1-\alpha)y_2)$ therefore if we choose our distribution such that $Pr[y=y_1] = \alpha$ and $Pr[y=y_2] = 1- \alpha$, it is easy to see that $\mathbb{E}_{\mathcal{D}}[\gamma(y)] \neq  \gamma(\mathbb{E}_{\mathcal{D}}[y])$.
\end{proof}

\begin{proof}[Proof of Theorem \ref{thm:two_stage_gap}]
We consider the case of $d=2$; the extension to arbitrarily many dimensions is trivial.
We first define our function $f$,
\begin{equation*}
    f(y,z)= \gamma (y_{1,1},y_{2,1})   f_{1}(z) + \gamma (y_{2,1},y_{2,2})   f_{2}(z).
\end{equation*}
We construct $\marginal$ and small $\epsilon>0$ such that $\mathbb{E}_{\marginal}[\gamma(y_{1,1},y_{1,2})] \neq \gamma(\mathbb{E}_{\marginal}[y_{1,1}],\mathbb{E}_{\marginal}[y_{1,2}])$. Such a distribution is guaranteed to exist by Lemma \ref{lemma:lin}.

Now there are two cases,
\begin{enumerate}
    \item  $\mathbb{E}_{\marginal}[\gamma(y_{1,1},y_{1,2})] < \gamma(\mathbb{E}_{\marginal}[y_{1,1}],\mathbb{E}_{\marginal}[y_{1,2}])$ where we choose point masses for our remaining two values such that $p(y_{2,1} = \mathbb{E}_{\marginal}[y_{1,1}] - \epsilon) = 1.0$ and $ p(y_{2,2} = \mathbb{E}_{\marginal}[y_{1,2}] - \epsilon) = 1.0$;
    \item  $\mathbb{E}_{\marginal}[\gamma(y_{1,1},y_{1,2})] > \gamma(\mathbb{E}_{\marginal}[y_{1,1}],\mathbb{E}_{\marginal}[y_{1,2}])$, where we choose point masses for our remaining two values such that $p(y_{2,1} = \mathbb{E}_{\marginal}[y_{1,1}] + \epsilon) = 1.0,$ and $ p(y_{2,2} = \mathbb{E}_{\marginal}[y_{1,2}] + \epsilon) = 1.0$.
\end{enumerate}

Without loss of generality, we assume 
that we are in the first case and  our constructed optimization problem will minimize $f$. The example works symmetrically in the second case where we can construct the optimization problem to maximize $f$.

We show that two-stage is suboptimal for the following optimization problem
\begin{align*}
    \min_{z} ~~&\mathbb{E}_{y \sim \marginal}[\gamma (y_{1,1},y_{1,2})   z_{1} + \gamma (y_{2,1},y_{2,2})   z_{2}]\quad \text{subject to } z_{1}+z_{2} \geq 1.
\end{align*}
It is easy to see that the optimal choice of $z$ is
$
z^{*}_{OPT} = \{z_{1}=1.0, z_{2}=0.0\},
$
which gives a loss of
$\loss(z^{*}_{OPT}) = \mathbb{E}_{\marginal}[\gamma(y_{1,1},y_{1,2})].$
However, since two-stage makes its choices with respect to $\gamma(\mathbb{E}_{\marginal}[y_{1,1}], \mathbb{E}_{\marginal}[y_{1,2}])$ 
and 
$\gamma(\mathbb{E}_{\marginal}[y_{2,1}], \mathbb{E}_{\marginal}[y_{2,2}]),$
it chooses the solution
$z^{*}_{\text{two-stage}} = \{z_{1}=0.0, z_{2}=1.0\},$
giving it a loss of
$\loss(z^{*}_{\text{two-stage}}) = \mathbb{E}_{\marginal}[\gamma(y_{2,1},y_{2,2})].$

If we choose a small enough $\epsilon$,
$\mathbb{E}_{\marginal}[\gamma(y_{1,1},y_{1,2})] < \mathbb{E}_{\marginal}[\gamma(y_{2,1},y_{2,2})] $, and hence $
\loss(z^{*}_{OPT}) < \loss(z^{*}_{\text{two-stage}} ).
$

This construction trivially extends to $d>2$ by making $f_{i}(z) = 0, \forall i > 2$.

\end{proof}

\begin{proof}[Proof of Theorem \ref{thm:linear}, Condition 1]

We know that $\mathcal{L}_{OPT}= \min \mathbb{E}_{y\sim \conditional{x}}[f(y,z)]$
and $\mathcal{L}_{\text{two stage}} = \min_{z} f(\mathbb{E}_{y\sim \conditional{x}}[y],z)]$

Plugging in our definition of $f(y,z)$ from Equation \ref{eqn:element_wise} we get
\begin{align*}
\mathcal{L}_{OPT}&= \min \mathbb{E}_{y\sim \conditional{x}}[ \sum_{i=1}^{d} \gamma (y_{i,1},y_{i,2}) \cdot f_{i}(z)]\\
&=\min \sum_{i=1}^{d} \mathbb{E}_{y\sim \conditional{x}} [\gamma (y_{i,1},y_{i,2}) \cdot f_{i}(z)]\\
&=\min \sum_{i=1}^{d} f_{i}(z) \cdot \mathbb{E}_{y\sim \conditional{x}} [ \gamma (y_{i,1},y_{i,2}) ].
\end{align*}
Since we know $\gamma$ is a linear function, by linearity of expectation
\begin{align*}
    \mathcal{L}_{OPT}&= \min_{z} \sum_{i=1}^{d} f_{i}(z) \cdot \gamma (\mathbb{E}_{y\sim \conditional{x}}[y_{i,1}],\mathbb{E}_{y\sim \conditional{x}}[y_{i,2}])\\
    &=\mathcal{L}_{\text{two stage}}.
\end{align*}
\end{proof}

\begin{proof}[Proof of Theorem \ref{thm:linear}, Condition 2]

We know that $\mathcal{L}_{OPT}= \min \mathbb{E}_{y\sim \conditional{x}}[f(y,z)]$
and $\mathcal{L}_{\text{two stage}} = \min_{z} f(\mathbb{E}_{y\sim \conditional{x}}[y],z)]$

Plugging in our definition of $f(y,z)$ from Equation \ref{eqn:element_wise} we get
\begin{align*}
\mathcal{L}_{OPT}&= \min_{z} \mathbb{E}_{y\sim \conditional{x}}[ \sum_{i=1}^{d} \gamma (y_{i,1},y_{i,2}) \cdot f_{i}(z)]\\
&=\min_{z} \sum_{i=1}^{d} \mathbb{E}_{y\sim \conditional{x}} [\gamma (y_{i,1},y_{i,2}) \cdot f_{i}(z)]\\
&=\min_{z} \sum_{i=1}^{d} f_{i}(z) \cdot \mathbb{E}_{y\sim \conditional{x}} [ \gamma (y_{i,1},y_{i,2}) ].
\end{align*}
Since $\gamma(y,y') = y\cdot y'$:
\begin{align*}
    \mathcal{L}_{OPT}&= \min_{z} \sum_{i=1}^{d} f_{i}(z) \cdot \mathbb{E}_{y\sim \conditional{x}} [ y_{i,1} \cdot y_{i,2}]\\
    &= \min_{z} \sum_{i=1}^{d} f_{i}(z) \cdot \int_{y\sim \conditional{x}} y_{i}P(y_{i,1}) \cdot y_{i,2}P(y_{i,2}|y_{i,1})).
\end{align*}
Since $P(y_{i,2}) = P(y_{i,2} | y_{i,1})$, by independence:
\begin{align*}
&\min_{z} \sum_{i=1}^{d} f_{i}(z) \cdot \int_{y\sim \conditional{x}} y_{i,1}P(y_{i,1}) \cdot \int_{y\sim \conditional{x}} y_{i,2}P(y_{i,2})\\
    &\min_{z} \sum_{i=1}^{d} f_{i}(z) \cdot \gamma (\mathbb{E}_{y\sim \conditional{x}}[y_{i,1}],\mathbb{E}_{y\sim \conditional{x}}[y_{i,2}])\\
    &=\mathcal{L}_{\text{two stage}}.
\end{align*}
\end{proof}

\begin{proof}[Generalization of Theorem \ref{thm:linear} for objective functions that are linear in $y$]

We know that $\mathcal{L}_{OPT}= \min \mathbb{E}_{y\sim \conditional{x}}[f(y,z)]$
and $\mathcal{L}_{\text{two stage}} = \min_{z} f(\mathbb{E}_{y\sim \conditional{x}}[y],z)]$. Since $f(y,z)$ is linear in $y$, by linearity of expectation we can move expectation inside $f$: 

\begin{align*}
\mathcal{L}_{OPT}&= \min_{z}  f(\mathbb{E}_{y\sim \conditional{x}} [y,z])\\
&= \min_{z}  f(\mathbb{E}_{y\sim \conditional{x}} [y],z) & \text{pulling out $z$ since expectation over $y$}\\
&=\mathcal{L}_{\text{two stage}}.
\end{align*}

\end{proof}

\end{document}